\documentclass[11pt]{article}
\usepackage{framed,epsf,latexsym,amsmath,amssymb,amscd,textcomp}
\usepackage[numbers]{natbib}                                 
\usepackage{graphicx}
\usepackage[pdftex,colorlinks]{hyperref}
\usepackage[margin=1.0in]{geometry}

\newtheorem{theorem}{Theorem}[section]
\newtheorem{lemma}[theorem]{Lemma}
\newtheorem{remark}[theorem]{Remark}

\newtheorem{example}[theorem]{Example}
\newtheorem{terminology}[theorem]{Terminology}
\newtheorem{counter-example}[theorem]{Counter example}

\newtheorem{assumption}[theorem]{Assumption}
\newtheorem{open question}[theorem]{Open question}
\newtheorem{corollary}[theorem]{Corollary}

\newcommand{\ignore}[1]{}

\newcommand{\ca}{{\cal A}}

\newcommand{\cd}{{\cal D}}

\newcommand{\cg}{{\cal G}}
\newcommand{\ch}{{\cal H}}

\newcommand{\cl}{{\cal L}}

\newcommand{\cx}{{\cal X}}

\newcommand{\cz}{{\cal Z}}

\newcommand{\csp}{\mathrm{CSP}}

\newcommand{\auto}{\mathrm{AUTO}}

\newcommand{\rand}{\mathrm{rand}}

\newcommand{\sat}{\mathrm{SAT}}
\newcommand{\dnf}{\mathrm{DNF}}

\newcommand{\val}{\mathrm{VAL}}

\newcommand{\proof}{{\par\noindent {\bf Proof}\space\space}}
\newcommand{\proofbox}{\hfill $\Box$}

\DeclareMathOperator{\Err}{Err}
\DeclareMathOperator{\poly}{poly}

\title{Complexity theoretic limitations on learning DNF's}

\author{Amit Daniely\thanks{Dept. of Mathematics, The Hebrew University, Jerusalem, Israel} \hspace{1cm}Shai Shalev-Shwartz\thanks{School of Computer Science and Engineering, The Hebrew University, Jerusalem, Israel} 
}

\begin{document}
\maketitle
\setcounter{page}{0}

\thispagestyle{empty}
\maketitle

\begin{abstract}
Using the recently developed framework of \cite{daniely2013average}, we show that under a natural assumption on the complexity of refuting random K-SAT
formulas, learning DNF formulas is hard. Furthermore, the same assumption implies the hardness of learning intersections of $\omega(\log(n))$ halfspaces, agnostically learning conjunctions, as well as 	virtually all (distribution free) learning problems that were previously shown hard (under complexity assumptions).
\end{abstract}

\newpage

\section{Introduction}
In the PAC learning model \cite{Valiant84}, a learner is
given an oracle access to randomly generated
samples $(X,Y)\in \cx\times\{0,1\}$ where $X$ is sampled from some
{\em unknown} distribution $\cd$ on $\cx$ and $Y=h^{*}(X)$ for some
{\em unknown} $h^{*} : \cx \to \{0,1\}$. It is assumed that $h^*$
comes from a predefined \emph{hypothesis class} $\ch$, consisting of $0,1$ valued
functions on $\cx$. The learning problem defined by $\ch$ is to find $h:\cx\to\{0,1\}$ that minimizes
$\Err_{\cd}(h):=\Pr_{X\sim\cd}(h(X)\not=h^*(X))$. For concreteness, we take
$\cx=\{\pm 1\}^n$, and say that the learning problem is tractable if there is an algorithm that on input
$\epsilon$, runs in time $\poly(n,1/\epsilon)$ and outputs, w.h.p., a
hypothesis $h$ with $\Err(h)\le\epsilon$.

Assuming $\mathbf{P}\ne\mathbf{NP}$, the status of most basic {\em
  computational} problems is fairly well understood. In a sharp
contrast, $30$ years after Valiant's paper, the status of most basic
{\em learning} problems is still wide open -- there is a huge gap
between the performance of best known algorithms and hardness
results (see \cite{daniely2013average}).  The main obstacle is the
ability of a learning algorithm to return a hypothesis which does not
belong to $\ch$ (such an algorithm is called {\em improper}). This
flexibility makes it very hard to apply reductions from
$\mathbf{NP}$-hard problems (again, see
\cite{daniely2013average}). Until recently, there was only a single
framework, due to Kearns and Valiant \cite{KearnsVa89}, to prove lower
bounds on learning problems.  The framework of \cite{KearnsVa89} makes
it possible to show that certain cryptographic assumptions imply	
hardness of certain learning problems.  As indicated above, the lower
bounds established by this method are very far from the
performance of best known algorithms.

In a recent paper \cite{daniely2013average} (see also
\cite{daniely2013more}) developed a new framework to prove 
hardness of learning based on hardness on average of
CSP problems. Yet, \cite{daniely2013average} were not able
to use their technique to establish hardness results that are based on
a natural assumption on a well studied problem. Rather, they made a
very general and strong hardness assumption, that is concerned with general $\csp$ problems, most of which were never studied explicitly.  This was recognized in \cite{daniely2013average} as the main weakness of
their approach, and therefore the main direction for further research.
In this paper we make a natural assumption on the extensively studied
problem of refuting random $K$-SAT instances, in the spirit of Feige's
assumption~\cite{Feige02}. Under this assumption, and using the framework
of \cite{daniely2013average}, we show:

\begin{enumerate}
\item\label{intro:1} Learning $\mathrm{DNF}$'s is hard.
\item\label{intro:2} Learning intersections of $\omega(\log(n))$ halfspaces is hard, even over the boolean cube.
\item\label{intro:3} Agnostically\footnote{See section \ref{sec:learning_background} for a definition of agnostic learning.} learning conjunctions is hard.
\item\label{intro:4} Agnostically learning halfspaces is hard, even over the boolean cube. 
\item\label{intro:5} Agnostically learning parities is hard, even when $\cd$ is uniform.
\item\label{intro:6} Learning finite automata is hard.
\end{enumerate}
We note that~\ref{intro:4},~\ref{intro:6} can be established under cryptographic assumptions, using the
cryptographic technique \citep{FeldmanGoKhPo06,KearnsVa89}. Also, \ref{intro:5}  follows from the hardness of learning parities with noise\footnote{Note that agnostically learning parities when $\cd$ is uniform is not equivalent to the problem that is usually referred as ``learning parities with noise", since in agnostic learning, the noise might depend on the instance.} \cite{blum2003noise}, which is often taken as a hardness assumption. As for \ref{intro:2}, the previously best lower bounds~\cite{KlivansSh06} only rule out learning intersections of polynomially many halfspaces, again under cryptographic assumptions.
To the best
of our knowledge, \ref{intro:1}-\ref{intro:6} implies the hardness of virtually all (distribution free) learning problems that were previously shown hard (under various complexity assumptions).

\subsection{The random $K$-SAT assumption} 
Unless we face a dramatic breakthrough in complexity theory,
it seems unlikely that hardness of learning can be established on standard complexity assumptions such as $\mathbf{P}\ne\mathbf{NP}$ (see \cite{ApplebaumBaXi08,daniely2013average}). Indeed, all currently known lower bounds are based on cryptographic assumptions. 
Similarly to Feige's paper~\cite{Feige02}, we rely here on the hardness of refuting random $K$-SAT formulas.
As cryptographic assumptions, our assumption asserts the hardness on average of a certain problem that have been resisted extensive attempts of attack during the last 50 years (e.g. \cite{davis1962machine,BeamePi96,BeameKaPiSa98,BenWi99,Feige02,feige2004easily,coja2004strong,coja2010efficient}).

Let $J=\{C_1,\ldots,C_m\}$ be a random $K$-SAT formula on $n$ variables. Precisely, each $K$-SAT constraint $C_i$ is chosen independently and uniformly from the collection of $n$-variate $K$-SAT constraints.
A simple probabilistic argument shows that for some constant $C$ (depending only on $K$), if $m\ge Cn$, then $J$ is not satisfiable w.h.p. The problem of {\em refuting random $K$-SAT formulas} (a.k.a. the problem of distinguishing satisfiable from random $K$-SAT formulas) seeks efficient algorithms that provide, for most formulas, a {\em refutation}. That is, a proof that the formula is not satisfiable.

Concretely, we say that an algorithm is able to refute random $K$-SAT instances with $m=m(n)\ge Cn$ clauses if on $1-o_n(1)$ fraction of the $K$-SAT formulas with $m$ constraints, it outputs ``unsatisfiable", while for {\em every} satisfiable $K$-SAT formula with $m$ constraints, it outputs ``satisfiable"\footnote{See a precise definition in section \ref{sec:CSP}}. Since such an algorithm never errs on satisfiable formulas, an output of ``unsatisfiable" provides a proof that the formula is not satisfiable.

The problem of refuting random $K$-SAT formulas has been extensively studied during the last 50 years.
It is not hard to see that the problem gets easier as $m$ gets larger. The currently best known algorithms~\cite{feige2004easily,coja2004strong,coja2010efficient} can only refute random instances with $\Omega\left(n^{\lceil\frac{K}{2}\rceil}\right)$ constraints for $K\ge 4$ and $\Omega\left(n^{1.5}\right)$ constraints for $K=3$.
In light of that, Feige~\cite{Feige02} made the assumption that for $K=3$, refuting random instances with $Cn$ constraints, for every constant $C$, is hard (and used that to prove hardness of approximation results). Here, we put forward the following assumption.

\begin{assumption}\label{hyp:only_sat}
Refuting random $K$-$\sat$ formulas with $n^{f(K)}$ constraints is hard for some $f(K)=\omega(1)$.
\end{assumption}
\begin{terminology}
Computational problem is {\em RSAT-hard} if its tractability refutes assumption \ref{hyp:only_sat}.
\end{terminology}
We outline below some evidence to the assumption, in addition to known algorithms' performance.

{\bf Hardness of approximation.} Define the {\em value}, $\val(J)$, of a $K$-SAT formula $J$ as the maximal fraction of constraints that can be
simultaneously satisfied. Hastad's celebrated result~\citep{haastad2001some} asserts that if $\mathbf{P}\ne \mathbf{NP}$, it is hard to distinguish
satisfiable $K$-SAT instances from instances with
$1-2^{-K}\le\val(J)\le 1-2^{-K}+\epsilon$. Since the value of a random formula is approximately $1-2^{-K}$, we can interpret Hastad's result as claiming that it is hard to distinguish satisfiable from ``semi-random" $K$-SAT formulas (i.e., formulas whose value is approximately the value of a random formula). Therefore, assumption \ref{hyp:only_sat} can be seen as a strengthening of Hastad's result.

{\bf Resolution lower bounds.} The length of resolution refutations of random $K$-SAT formulas have been extensively studied (e.g.~\cite{haken1985intractability,BeamePi96,BeameKaPiSa98,BenWi99}). It is known (theorem 2.24 in~\cite{ben2001expansion}) that random formulas with $n^{\frac{K}{2}-\epsilon}$ constraints only have exponentially long resolution refutations. This shows that a large family of algorithms (the so-called Davis-Putnam algorithms~\cite{davis1962machine}) cannot efficiently refute random formulas with $n^{\frac{K}{2}-\epsilon}$ constraints. These bounds can also be taken as an indication that random instances do not have short refutations in general, and therefore
hard to refute.

{\bf Hierarchies lower bounds.} Another family of algorithms whose performance has been analyzed are convex relaxations \cite{buresh2003rank, schoenebeck2008linear, alekhnovich2005towards}.
In \cite{schoenebeck2008linear} it is shown that relaxations in the Lasserre hierarchy with sub-exponential many constraints cannot refute random formulas with $n^{\frac{K}{2}-\epsilon}$ constraints.

\subsection{Results}

{\bf Learning DNF's.}
A {\em $\dnf$ clause} is a conjunction of literals. A {\em $\dnf$ formula} is a disjunction of $\dnf$ clauses. Each $\dnf$ formula over $n$ variables naturally induces a function on $\{\pm 1\}^n$.  
The {\em size} of a $\dnf$ clause is the the number of literals, and the size of a $\mathrm{DNF}$ formula is the sum of the sizes of its clauses. For $q:\mathbb N\to\mathbb N$, denote by
$\mathrm{DNF}_{q(n)}$ the class of functions over $\{\pm
1\}^n$ that are realized by $\mathrm{DNF}$s of size $\le q(n)$.  Also, $\mathrm{DNF}^{q(n)}$ is the class of
functions that are realized by
$\mathrm{DNF}$ formulas with $\le q(n)$ clauses. Since each clause
is of size at most $n$, $\mathrm{DNF}^{q(n)}\subset
\mathrm{DNF}_{nq(n)}$.

Learning
hypothesis classes consisting of poly sized $\dnf$'s formulas has been a major effort in computational learning theory (e.g.~\cite{Valiant84,klivans2001learning,LinialMaNi89,mansour1995nlog}). Already in
Valiant's paper \citep{Valiant84}, it is shown that for every constant
$q$, $\dnf$-formulas with $\le q$ clauses
can be learnt efficiently. As for lower bounds, {\em properly} learning $\dnf$'s is known to be
hard \citep{PittVa88}. Yet, hardness of improperly learning
$\dnf$'s formulas has remained a major open question. Here we show:
\begin{theorem}\label{thm:dnf_few_clauses}
If $q(n)=\omega(\log(n))$ then learning $\mathrm{DNF}^{q(n)}$ is RSAT-hard.
\end{theorem}
Since $\mathrm{DNF}^{q(n)}\subset \mathrm{DNF}_{nq(n)}$, we immediately conclude that learning $\dnf$'s of size, say, $\le n\log^2(n)$, is RSAT-hard. By a simple scaling argument (e.g. \cite{daniely2013average}), we obtain an even stronger result:
\begin{corollary}\label{cor:dnf_small}
For every $\epsilon>0$, it is RSAT-hard to learn $\dnf_{n^\epsilon}$.
\end{corollary}

\begin{remark}\label{rem:boosting}
By boosting results \cite{Schapire89}, hardness of improper learning is automatically very strong quantitatively. Namely, for every $c>0$, it is hard to find a classifier with error $\le \frac{1}{2}-\frac{1}{n^c}$. Put differently, making a random guess on each example, is essentially optimal.
\end{remark}
{\bf Additional results.}
Theorem \ref{thm:dnf_few_clauses} implies the hardness of several problems, in addition to DNFs.

\begin{theorem}\label{thm:intersection}
Learning intersections of $\omega(\log(n))$ halfsapces over $\{\pm 1\}^n$ is RSAT-hard.
\end{theorem}

\begin{theorem}\label{thm:monomials}
Agnostically learning conjunctions is RSAT-hard.
\end{theorem}

\begin{theorem}\label{thm:halfspaces}
Agnostically learning halfspaces over $\{\pm 1\}^n$ is RSAT-hard.
\end{theorem}

\begin{theorem}\label{thm:parity}
Agnostically learning parities\footnote{A parity is any hypothesis of the form $h(x)=\Pi_{i\in S}x_i$ for some $S\subset [n]$.} is RSAT-hard, even when the marginal distribution is uniform on $\{\pm 1\}^n$.
\end{theorem}

\begin{theorem}\label{thm:auto}
For every $\epsilon>0$, learning automata  of size $n^{\epsilon}$ is RSAT-hard.
\end{theorem}
Theorem \ref{thm:intersection} is a direct consequence of theorem \ref{thm:dnf_few_clauses}, as a DNF formula with $q(n)$ clauses is an intersection of $q(n)$ halfspaces. 
Theorem \ref{thm:monomials} follows from theorem \ref{thm:dnf_few_clauses}, as learning DNFs can be reduced to agnostically learning conjunctions~\cite{LeeBaWi96}.
Theorem \ref{thm:halfspaces} follows from theorem \ref{thm:monomials}, as conjunctions are a subclass of halfspaces.
Theorem \ref{thm:parity} follows from theorem \ref{thm:dnf_few_clauses} and ~\cite{FeldmanGoKhPo06}, who showed that learning DNFs can be reduced to agnostically learning parities over the uniform distribution.
Theorem \ref{thm:auto} follows from theorem \ref{thm:dnf_few_clauses} by a simple reduction (see section \ref{sec:auto}).

\subsection{Related work}
As indicated above, hardness of learning is traditionally established based on cryptographic assumptions. The first such result follows from \cite{GoldreichGoMi86}, and show that if one-way functions exist, than it is hard to learn polynomial sized circuits. To prove lower bounds on simpler hypothesis classes, researchers had to rely on more concrete hardness assumptions. Kearns and Valiant \cite{KearnsVa89} were the first to prove such results. They showed that assuming the hardness of various cryptographic problems (breaking RSA, factoring Blum integers and detecting quadratic residues), it is hard to learn automata, constant depth threshold circuits, $\log$-depth circuits and boolean formulae. Kharitanov~\cite{Kharitonov93} showed, under a relatively strong assumption on the complexity of factoring random Blum integers, that learning constant depth circuits (for unspecified constant) is hard. Klivans and Sherstov \cite{KlivansSh06} showed that, under the hardness of the shortest vector problem, learning intersections of polynomially many halfspaces is hard. By \cite{BlumKa97}, it also follows that agnostically learning halfspaces is hard. Hardness of agnostically learning halfspaces also follows from the hardness of learning parities with noise \cite{KalaiKlMaSe05}.

There is a large body of work on various variants of the standard (improper and distribution free) PAC model. Hardness of proper learning, when the leaner must return a hypothesis from the learnt class, in much more understood (e.g. \cite{khot2008hardness, khot2011hardness, GuruswamiRa06, FeldmanGoKhPo06, PittVa88}). Hardness of learning with restrictions on the distribution were studied in, e.g., \cite{klivans2014embedding, KalaiKlMaSe05, Kharitonov93}. Hardness of learning when the learner can ask the label of unseen examples were studied in, e.g., \cite{AngluinKh91, Kharitonov93}.

Lower bounds using the technique we use in this paper initiated in \cite{daniely2013more, daniely2013average}. In \cite{daniely2013more} it was shown, under Feige's assumption, that if the number of examples is limited (even tough information theoretically sufficient), then learning halfspaces over sparse vectors is hard. The full methodology we use here presented in \cite{daniely2013average}. They made a strong and general assumption, that says, roughly, that for every random $\mathrm{CSP}$ problem, if the number of random constraints is too small to provide short resolution proofs, then the SDP relaxation of \cite{raghavendra2008optimal} has optimal approximation ratio. Under this assumption they concluded hardness results  that are similar to the results presented here.

\section{Preliminaries}
\subsection{PAC Learning}\label{sec:learning_background}
A {\em hypothesis class}, $\ch$, is a series of collections of functions $\ch_n\subset \{0,1\}^{\cx_n},\;n=1,2,\ldots$. We often abuse notation and identify $\ch$ with $\ch_n$. The instance spaces $\cx_n$ we consider are $\{\pm 1\}^n$, $\{0,1\}^n$ or $\cx_{n,K}$ (see section \ref{sec:CSP}).  
Distributions on $\cz_n:=\cx_n\times\{0,1\}$ are denoted $\cd_n$.
The error of $h:\cx_n\to\{0,1\}$  is $\Err_{\cd_n}(h)=\Pr_{(x,y)\sim\cd_n }\left(h(x)\ne y\right)$. For a class $\ch_n$, we let $\Err_{\cd_n}(\ch_n)=\min_{h\in\ch_n}\Err_{\cd_n}(h)$. We say that 
$\cd_n$ is \emph{realizable} by $h$ (resp. $\ch_n$) if $\Err_{\cd_n}(h)=0$ (resp. $\Err_{\cd_n}(\ch_n)=0$). 
A {\em sample} is a sequence $S=\{(x_1,y_1),\ldots (x_m,y_m)\}\in\cz^m_n$. The {\em empirical error} of $h:\cx_n\to\{0,1\}$ on $S$ is $\Err_{S}(h)=\frac{1}{m}\sum_{i=1}^m1(h(x_i)\ne y_i)$, while the empirical error of $\ch_n$ on $S$ is $\Err_{S}(\ch_n)=\min_{h\in\ch_n}\Err_S(h)$.
We say that $S$ is \emph{realizable} by $h$ (resp. $\ch_n$) if $\Err_S(h)=0$ (resp. $\Err_S(\ch_n)=0$). 

A {\em learning algorithm}, $\cl$, obtains an error, confidence and complexity parameters $0<\epsilon<1$, $0<\delta<1$, and $n$, as well as oracle access to examples from unknown distribution $\cd_n$ on $\cz_n$. It should output a (description of) hypothesis $h:\cx_n\to\{0,1\}$. We say that $\cl$ {\em (PAC) learns} $\ch$ if, for every realizable $\cd_n$, w.p. $\ge 1-\delta$, $\cl$ outputs a hypothesis with error $\le \epsilon$.
We say that $\cl$ {\em agnostically learns} $\ch$ if, for every $\cd_n$, w.p. $\ge 1-\delta$, $\cl$ outputs a hypothesis with error $\le \Err_{\cd_n}(\ch)+\epsilon$. 
We say that $\cl$ is {\em efficient} if it runs in time $\poly(n,1/\epsilon,1/\delta)$, and outputs a hypothesis that can be evaluated in time $\poly(n,1/\epsilon,1/\delta)$. Finally, $\cl$ is {\em proper} if it always outputs a hypothesis in $\ch$. Otherwise, we say that $\cl$ is {\em improper}.

\subsection{Random Constraints Satisfaction Problems}\label{sec:CSP}
Let $\cx_{n,K}$ be the collection of {\em (signed) $K$-tuples}, that is, vectors $x=[(\alpha_1,i_1),\ldots,(\alpha_K,i_K)]$ for $\alpha_1,\ldots,\alpha_K\in \{\pm 1\}$ and distinct $i_1,\ldots,i_K\in [n]$. For $j\in [K]$ we denote  $x(j)=(x^1(j),x^2(j))=(\alpha_j,i_j)$.
Each $x\in \cx_{n,K}$ defines a function $U_x:\{\pm 1\}^n\to\{\pm 1\}^K$ by $U_x(\psi)=(\alpha_1\psi_{i_1},\ldots,\alpha_K\psi_{i_K})$.

Let $P:\{\pm 1\}^K\to \{0,1\}$ be some predicate. A {\em
  $P$-constraint} with $n$ variables is a function $C:\{\pm
1\}^n\to\{0,1\}$ of the form $C(x)=P\circ U_x$ for some $x\in\cx_{n,K}$.
An instance to the {\em CSP problem} $\csp(P)$ is
a {\em $P$-formula}, i.e., a collection $J=\{C_1,\ldots,C_m\}$ of $P$-constraints (each is specified by a $K$-tuple). The
goal is to find an assignment $\psi\in \{\pm 1\}^n$ that maximizes
the fraction of satisfied constraints (i.e.,
constraints with $C_i(\psi)=1$). We will allow CSP problems where $P$ varies with $n$ (but is still fixed for every $n$). For example, we can look of the $\lceil\log(n)\rceil$-SAT problem.

We will often consider the problem of distinguishing satisfiable from random $P$ formulas (a.k.a. the problem of refuting random $P$ formulas).
Concretely, for $m:\mathbb{N}\to\mathbb{N}$, we say that the problem $\csp^{\rand}_{m(n)}(P)$ is easy, if there exists an efficient randomized algorithm, $\ca$, such that:
\begin{itemize}
\item If $J$ is a satisfiable instance to $\csp(P)$ with $n$ variables and $m(n)$ constraints, then
\[
\Pr_{\text{coins of }\ca}\left(\ca(J)=\text{``satisfiable"}\right)\ge\frac{3}{4}
\]  
\item If $J$ is a random\footnote{To be precise, in a random formula with $n$ variable and $m$ constraints, the $K$-tuple defining each constraint is chosen uniformly, and independently from the other constraints.} instance to $\csp(P)$ with $n$ variables and $m(n)$ constraints then, with probability $1-o_n(1)$ over the choice of $J$,
\[
\Pr_{\text{coins of }\ca}\left(\ca(J)=\text{``random"}\right)\ge \frac{3}{4}~.
\]  
\end{itemize}

\subsection{The methodology of \cite{daniely2013average}}\label{sec:methodology}
In this section we briefly survey the technique of \cite{daniely2013average, daniely2013more} to prove hardness of improper learning. 
Let $\cd=\{\cd^{m(n)}_n\}_{n}$ be a polynomial ensemble of distributions, that is, $\cd^{m(n)}_n$ is a distribution on $\cz_n^{m(n)}$ and $m(n)\le\poly(n)$. Think of $\cd^{m(n)}_n$ as a distribution that generates samples that are far from being realizable.
We say that it is hard to distinguish realizable from $\cd$-random samples if there is no efficient randomized algorithm $\ca$ with the following properties:
\begin{itemize}
\item For every realizable sample $S\in \cz^{m(n)}_n$,
\[\Pr_{\text{internal coins of }\ca}\left(\ca(S)=``realizable"\right)\ge \frac{3}{4}~.\]
\item If $S\sim \cd_n^{m(n)}$, then with probability $1-o_n(1)$ over the choice of $S$, it holds that 
\[\Pr_{\text{internal coins of }\ca}\left(\ca(S)=``unrelizable"\right)\ge \frac{3}{4}~.\]
\end{itemize}
For $p:\mathbb N\to(0,\infty)$ and $1>\beta>0$, we say that $\cd$ is {\em $(p(n),\beta)$-scattered} if, for large enough $n$, it holds that for every function $f:\cx_n\to\{0,1\}$, $
\Pr_{S\sim\cd^{m(n)}_n}\left(\Err_{S}(f)\le \beta\right)\le 2^{-p(n)}$.
\begin{example}\label{examp:1}
Let $\cd_n$ be a distribution over $\cz_n$ such that if $(x,y)\sim\cd_n$, then $y$ is a Bernoulli r.v. with parameter $\frac{1}{2}$, independent from $x$.
Let $\cd^{m(n)}_n$ be the distribution over $\cz_n^{m(n)}$ obtained by taking $m(n)$ independent examples from $\cd_n$. 
For $f:\cx_n\to\{0,1\}$, $\Pr_{S\sim\cd^{m(n)}_n}\left(\Err_{S}(f)\le \frac{1}{4}\right)$ is the probability of getting at most $\frac{m(n)}{4}$ heads in $m(n)$ independent tosses of a fair coin. By Hoeffding's bound, this probability is $\le 2^{-\frac{1}{8} m(n)}$. Therefore, $\cd=\{\cd^{m(n)}_n\}_{n}$ is $\left(\frac{1}{8} m(n),1/4\right)$-scattered.
\end{example}
Hardness of distinguishing realizable from scattered samples turns out to imply hardness of learning.
\begin{theorem}\cite{daniely2013average}\label{thm:basic_realizable}
Every hypothesis class that satisfies the following condition is not efficiently learnable. There exists $\beta>0$ such that for every $d>0$ there is an $(n^d,\beta)$-scattered ensemble $\cd$ for which it is hard to distinguish between a $\cd$-random sample and a realizable sample.
\end{theorem}
The basic observation of \cite{daniely2013average, daniely2013more} is that an efficient algorithm, running on a very scattered sample, will return a bad hypothesis w.h.p. The reason is that the output classifier has a short description, given by the polynomially many examples the algorithm uses. Hence, the number of hypotheses the algorithm {\em might return} is limited. Now, since the sample is scattered, all these hypotheses are likely to perform purely. Based on that observation, efficient learning algorithm can efficiently distinguish realizable from scattered samples: We can simply run the algorithm on the given sample to obtain a classifier $h$. Now, if the sample is realizable, $h$ will perform well. Otherwise, if the sample is scattered, $h$ will perform purely. Relying on that, we will be able to distinguish between the two cases. For completeness, we include the proof of theorem \ref{thm:basic_realizable} in section \ref{sec:proof_thm_basic}.

\section{Proof of theorem \ref{thm:dnf_few_clauses} }
\subsection{An overview}
Intuitively, the problem of distinguishing satisfiable from random formulas is similar to the problem of distinguishing realizable from random samples. In both problems, we try to distinguish rare and structured instances from very random and ``messy" instances. 
The course of the proof is to reduce the first problem to the second. Concretely, we reduce the problem $\csp_{n^d}^{\rand}(\sat_K)$ to the problem of distinguishing realizable (by $\dnf^{q(n)}$) samples from $(n^{d-2},\frac{1}{4})$-scattered samples. With such a reduction at hand, assumption \ref{hyp:only_sat} and theorem \ref{thm:basic_realizable}, implies theorem \ref{thm:dnf_few_clauses}.

\subsubsection*{CSP problems as learning problems}
The main conceptual idea is to interpret CSP problems as learning problems. Let $P:\{\pm 1\}^K\to\{0,1\}$ be some predicate. Every $\psi\in \{\pm 1\}^n$ naturally defines  $h_\psi:\cx_{n,K}\to \{0,1\}$, by mapping each $K$-tuple $x$ to the truth value of the corresponding constraint, given the assignment $\psi$. Namely, $h_{\psi}(x)=P\circ U_x(\psi)$.
Finally, let $\ch_{P}\subset \{0,1\}^{\cx_{n,K}}$ be the hypothesis class $\ch_{P}=\{h_\psi\mid \psi\in\{\pm 1\}^n\}$.

The problem $\csp(P)$ can be now formulated as follows. Given $x_1,\ldots,x_m\in \cx_{n,K}$, find $h_\psi\in \ch_P$ with minimal error on the sample $(x_1,1),\ldots,(x_m,1)$. Now, the problem $\csp^{\rand}_{m(n)}(P)$ is the problem of distinguishing a realizable sample from a random sample $(x_1,1),\ldots,(x_m,1)\in \cx_{n,K}\times\{0,1\}$ where the different $x_i$'s where chosen independently and uniformly from $\cx_{n,K}$.

The above idea alone, applied on the problem $\csp_{m(n)}^{\rand}(\sat_K)$ (or other problems of the form $\csp_{m(n)}^{\rand}(P)$), is still not enough to establish theorem \ref{thm:dnf_few_clauses}, due to the two following points:
\begin{itemize}
\item In the case that sample $(x_1,1),\ldots,(x_m,1)$ is random, it is, in a sense, ``very random". Yet, it is not scattered at all! Since all the labels are $1$, the constant function $1$ realizes the sample.
\item We must argue about the class $\dnf^{q(n)}$ rather than the class $\ch_{P}$.
\end{itemize}
Next, we explain how we address these two points.

\subsubsection*{Making the sample scattered}
To address the first point, we reduce $\csp^{\rand}_{n^d}(\sat_K)$ to a problem of the following form. For a predicate 
$P:\{\pm 1\}^K\to\{0,1\}$ we 
denote by $\csp(P,\neg P)$ the problem whose instances are collections, $J$, of constraints, each of which is either $P$ or $\neg P$ constraint, and the goal is to maximize the number of satisfied constraints.
Denote by
$\csp^{\rand}_{m(n)}(P,\neg P)$ the problem of distinguishing\footnote{As in $\csp^{\rand}_{m(n)}(P)$, in order to succeed, and algorithm must return ``satisfiable" w.p. $\ge\frac{3}{4}$ on every satisfiable formula and ``random" w.p. $\ge\frac{3}{4}$ on $1-o_n(1)$ fraction of the random formulas.} satisfiable from random formulas with $n$ variables and $m(n)$ constraints. Here, in a random formula, each constraint is chosen w.p. $\frac{1}{2}$ to be a uniform $P$ constraint and w.p. $\frac{1}{2}$ a uniform $\neg P$ constraint.

The advantage of the problem $\csp_{m(n)}^\rand(P,\neg P)$ is that in the ``learning formulation" from the previous section, it is the problem of distinguishing a realizable sample from a sample $(x_1,y_1),\ldots,(x_m,y_m)\in \cx_{n,K}\times\{0,1\}$ where the pairs $(x_i,y_i)$ where chosen at random, independently and uniformly. As explained in example \ref{examp:1}, this sample is $(\frac{1}{8}m(n),\frac{1}{4})$-scattered.

We will consider the predicate $T_{K,M}:\{0,1\}^{KM}\to\{0,1\}$ defined by
\[
T_{K,M}(z)=
\left(z_1\vee\ldots\vee z_K\right)\wedge
\left(z_{K+1}\vee\ldots\vee z_{2K}\right)
\wedge
\ldots
\wedge
\left(z_{(M-1)K+1}\vee\ldots\vee z_{MK}\right)~.
\]
We reduce the problem $\csp^{\rand}_{n^d}(\sat_K)$ to $\csp^{\rand}_{n^{d-1}}(T_{K,q(n)},\neg T_{K,q(n)})$. This is done in two steps. First, we reduce $\csp^{\rand}_{n^d}(\sat_K)$ to $\csp^{\rand}_{n^{d-1}}(T_{K,q(n)})$. This is done as follows. Given an instance $J=\{C_1,\ldots,C_{n^d}\}$ to $\csp(\sat_K)$, by a simple greedy procedure, we try to find $n^{d-1}$ disjoint subsets $J'_1,\ldots,J'_{n^{d-1}}\subset J$, such that for every $t$, $J'_t$ consists of $q(n)$ constraints and each variable appears in at most one of the constraints in $J'_t$. Now, from every $J'_t$ we construct $T_{K,q(n)}$-constraint that is the conjunction of all constraints in $J'_t$. As we show, if $J$ is random, this procedure will succeed w.h.p. and will produce a random $T_{K,q(n)}$-formula. If $J$ is satisfiable, this procedure will either fail or produce a satisfiable
$T_{K,q(n)}$-formula. 

The second step is to reduce $\csp^{\rand}_{n^{d-1}}(T_{K,q(n)})$ to $\csp^{\rand}_{n^{d-1}}(T_{K,q(n)},\neg T_{K,q(n)})$. This is done by replacing each constraint, w.p. $\frac{1}{2}$, with a random $\neg P$ constraint. Clearly, if the original instance is a random instance to $\csp^{\rand}_{n^{d-1}}(T_{K,q(n)})$, the produced instance is a random instance to $\csp^{\rand}_{n^{d-1}}(T_{K,q(n)},\neg T_{K,q(n)})$. Furthermore, if the original instance is satisfied by the assignment $\psi\in\{\pm 1\}^n$, the same $\psi$, w.h.p., will satisfy all the new constraints. The reason is that the predicate $\neg T_{K,q(n)}$ is positive on almost all inputs -- namely, on $1-\left(1-2^{-K}\right)^{q(n)}$ fraction of the inputs. Therefore the probability that a random $\neg T_{K,q(n)}$-constraint is satisfied by $\psi$ is $1-\left(1-2^{-K}\right)^{q(n)}$, and hence, the probability that all new constraints are satisfied by $\psi$ is $\ge 1-n^{d-1}\left(1-2^{-K}\right)^{q(n)}$. Now, since $q(n)=\omega(\log(n))$, the last probability is $1-o_n(1)$.

\subsubsection*{Reducing $\ch_{P}$ to $\dnf^{q(n)}$}
To address the second point, we will realize $\ch_{\neg T_{K,q(n)}}$ by the class $\dnf^{q(n)}$. 
More generally, we will show that for every predicate $P:\{\pm 1\}^K\to\{0,1\}$ expressible by a DNF formula with $T$ clauses, $\ch_P$ can be realized by DNF formulas with $T$ clauses (note that for $\neg T_{K,q(n)}$, $T=q(n)$).

We first note that hypotheses in $\ch_P$ are defined over signed $K$-tuples, while $\dnf$'s are defined over the boolean cube. To overcome that, we will construct an (efficiently computable) mapping $g:\cx_{n,K}\to \{\pm 1\}^{2Kn}$, and show that each $h\in\ch_P$ is of the form $h=h'\circ g$ for some DNF formula $h'$ with $T$ clauses and $2Kn$ variables. Besides ``fixing the domain", $g$ will have additional role -- we will choose an expressive $g$, which will help us to realize hypotheses in $\ch_P$. In a sense, $g$ will be a first layer of computation, that is the same for all $h\in\ch_P$ (and therefore we do not ``pay" for it).

We will 
group the coordinates of vectors in $\{\pm 1\}^{2Kn}$
into $2K$ groups, corresponding to $P$'s literals, and
index them by $[K]\times \{\pm 1\}\times[n]$. For $x=[(\alpha_1,i_1),\ldots,(\alpha_K,i_K)]\in \cx_{n,K}$, $g(x)$ will be the vector whose all coordinates are $1$, except that for $j\in [K]$, the $(j,-\alpha_j,i_j)$ coordinate is $-1$. 

Now, given $\psi\in\{\pm 1\}^n$, we show that $h_{\psi}:\cx_{n,K}\to \{0,1\}$ equals to $h\circ g$ for a DNF formula $h$ with $T$ clauses.
Indeed, suppose that $P(x)=C_1(x)\vee\ldots\vee C_T(x)$ is a DNF representation of $P$. It is enough to show that for every $C_r(z)=(-1)^{\beta_1}z_{j_1}\wedge\ldots\wedge (-1)^{\beta_l}z_{j_l}$ there is a conjunction of literals  $h_r:\{\pm 1\}^{2Kn}\to \{0,1\}$ such that for all $x=[(\alpha_1,i_1),\ldots,(\alpha_K,i_K)]\in \cx_{n,K}$, $h_r(g(x))=C_r(U_{x}(\psi))$.
To see that such $h_r$ exists, note that $C_r(U_{x}(\psi))=1$ if and only if, for every $1\le\tau\le l$, all the values in $g(x)$ in the coordinates of the form $(j_{\tau},\psi_i(-1)^{\beta_\tau},i)$ are $1$. 

\subsection{From $\csp_{n^d}^{\rand}(\sat_K)$ to $\csp_{n^{d-1}}^{\rand}(T_{K,\frac{n}{\log(n)}})$}
\begin{lemma}\label{lem:sat_to_T}
The problem $\csp_{n^d}^{\rand}(\sat_K)$ can be  reduced to $\csp_{n^{d-1}}^{\rand}(T_{K,M})$ for any $M\le \frac{n}{\log(n)}$.
\end{lemma}
It will be convenient to use the following strengthening of Chernoff's bound, recently proved (with a {\em very} simple proof) by Linial and Luria~\cite{LinLur14}
\begin{theorem}\cite{LinLur14}\label{thm:chern}
Let $X_1\ldots,X_n$ be indicator random variables such that for every 
$S\subset [n]$, $\Pr\left(\forall i\in S,\;X_i=1\right)\le \alpha^{|S|}$. Then, for every $\beta>\alpha$,
\[
\Pr\left(\frac{1}{n}\sum_{i=1}^nX_i\ge \beta\right)\le \exp(-D(\beta||\alpha)n)\le \exp(-2(\beta-\alpha)^2n)
\]
\end{theorem}

\begin{proof}
For simplicity, we assume that $M=\frac{n}{\log(n)}$.
Suppose toward a contradiction that $\csp_{n^{d-1}}^{\rand}(T_{K,\frac{n}{\log(n)}})$ can be efficiently solved using an algorithm $\ca$. Consider the following algorithm, $\ca'$, to $\csp_{n^d}^{\rand}(\sat_K)$. On the input $J=\{C_1,\ldots, C_{n^d}\}$,
\begin{enumerate}
\item Partition the constraints in $J$ into $n^{d-1}$ blocks, $\{C_{t+1},\ldots,C_{t+n}\},\;\;t=1,2,\ldots,n^{d-1}$.
\item For $t=1,\ldots,n^{d-1}$
\begin{enumerate}
\item Let $J'_t=\emptyset$.
\item For $r=1,\ldots,n$
\begin{enumerate}
\item If $|J'_t|<\frac{n}{\log(n)}$ and, for all $C\in J'_t$, the set variables appearing in $C_{t+r}$ is disjoint from the set of variables appearing in $C$, add $C_{t+r}$ to $J'_t$.
\end{enumerate}
\item \label{step:1} If $|J'_t|<\frac{n}{\log(n)}$, return ``satisfiable".
\item Let $C'_t$ be the $T_{K,\lceil\frac{n}{\log(n)}\rceil}$-constraint which is the conjunction of all the constraints in $J'_t$.
\end{enumerate}
\item \label{step:2} Run $\ca$ on the instance $J'=\{C'_1,\ldots, C'_{n^{d-1}}\}$ and return the same answer as $\ca$.
\end{enumerate}
Next, we reach a contradiction as we prove that $\ca'$ solves the problem $\csp_{n^{d}}^{\rand}(\sat_K)$. 
First, suppose that the input, $J$, is satisfiable. Then, either $\ca'$ will return ``satisfiable" in step \ref{step:1} or, will run $\ca$ on $J'$. It is not hard to see that $J'$ is satisfiable as well, and therefore, $\ca$ (and therefore $\ca'$) will return ``satisfiable" w.p. $\ge \frac{3}{4}$.

Suppose now that $J$ is random. First, we claim that $\ca'$ will reach \ref{step:2} w.p. $\ge 1-o_n(1)$. Indeed, we will show that for large enough $n$ and any fixed $t$, the probability of exiting at step \ref{step:1} is $\le \exp\left(-\left(\frac{1}{2^{2K+5}K}\right)^2n\right)$, from which it follows that the probability of exiting at step \ref{step:1} for some $t$ is $o_n(1)$.
To show that, let $X_r,\;r=1,\ldots,n$ be the indicator r.v. that is $1$ if and only if one of the variables appearing in $C_{t+r}$ also appears in one of $C_{t+1},\ldots,C_{t+r-1}$. Denote also $\bar{X}_r=1-X_r$

Let $n'=\lfloor\frac{n}{2K}\rfloor$.
It is enough to show that $\sum_{r=1}^{n'}\bar{X}_r\ge \frac{n}{\log(n)}$ w.p. $\ge 1- \exp\left(-\left(\frac{1}{2^{2K+5}K}\right)^2n\right)$. Indeed, for every fixed $r\in [n']$, since the number of variables appearing in 
$C_{t+1},\ldots,C_{t+r-1}$ is $\le \frac{n}{2}$, the probability that $X_r=1$ is $\le 1-2^{-K}$, even if we condition on  $X_1,\ldots,X_{r-1}$. Hence, the probability that any fixed $u$ variables out of $X_1,\ldots,X_{n'}$ are all $1$ is $\le \left(1-2^{-K}\right)^u$. By theorem \ref{thm:chern},
\[
\Pr\left(\frac{1}{n'}\sum_{i=1}^{n'} X_i\ge 1-2^{-K}+2^{-{K+1}}\right)\le \exp\left(-2\left(2^{-{K+1}}\right)^2n'\right)\le \exp\left(-\left(\frac{1}{2^{2K+5}K}\right)^2n\right)~.
\]
It follows that w.p. $\ge1- \exp\left(-\left(\frac{1}{2^{2K+5}K}\right)^2n\right)$, $\sum_{r=1}^{n'}\bar{X}_r\ge \frac{n'}{2^{K+1}}$, and the claim follows as for sufficiently large $n$, $\frac{n'}{2^{K+1}}\ge \frac{n}{\log(n)}$.
Finally, it is not hard to see that, conditioning on the event that the algorithm reaches step \ref{step:2}, $J'$ is random as well, and therefore w.p. $\ge 1-o_n(1)$ over the choice of $J$, $\ca$ (and therefore $\ca'$) will return ``random" w.p. $\ge \frac{3}{4}$ over its internal randomness.

\end{proof}

\subsection{From $\csp_{n^{d}}^{\rand}(T_{K,M})$ to $\csp_{n^{d}}^{\rand}(T_{K,M},\neg T_{K,M})$}

\begin{lemma}\label{lem:T_to_double_T}
For any fixed $K$ and $M\ge 2^{K+2}\cdot\log(m(n))$, the problem $\csp_{m(n)}^{\rand}(T_{K,M})$ can be efficiently reduced to the problem $\csp_{m(n)}^{\rand}(T_{K,M},\neg T_{K,M})$
\end{lemma}
\begin{proof}
Given an instance $J=\{C_1,\ldots,C_m\}$ to $\csp(T_{K,M})$, the reduction will generate an instance to $\csp(T_{K,M},\neg T_{K,M})$ as follows. For each $C_i$, w.p. $\frac{1}{2}$, we substitute $C_i$ by a random $\neg T_{K,M}$ constraint. Clearly, if $J$ is a random formula, then the produced formula is a valid random formula to $\csp_{m(n)}^{\rand}(T_{K,M},\neg T_{K,M})$. It remains to show that if $J$ 
is satisfiable, then so is $J'$. Indeed, let $\psi\in \{\pm 1\}^n$ be a satisfying assignment to $J$. It is enough to show that w.p. $\ge \frac{1}{m(n)}$ $\psi$ 
satisfies all the new $\neg T_{K,M}$-
constraints. However, since $\left|\left(\neg T_{K,M}\right)^{-1}(0)\right|
=\left(2^K-1\right)^{M}
=\left(1-2^{-K}\right)^{M}\cdot 2^{MK}$, the probability that a single random constraint is not satisfied is $ \left(1-2^{-K}\right)^{M}$. It follows that the probability that one of the random $\neg T_{K,M}$ constraints in $J'$ is not satisfiable by $\psi$ is $\le m(n) \left(1-2^{-K}\right)^{M}$. Finally, we have $m(n) \left(1-2^{-K}\right)^{M}\le \frac{1}{m(n)}$ since,
\begin{eqnarray*}
\log\left(m(n)\left(1-2^{-K}\right)^{M}\right)&=&\log(m(n))-M\log\left(\frac{1}{1-2^{-K}}\right)
\\
&=&\log(m(n))-M\log\left(1+\frac{2^{-K}}{1-2^{-K}}\right)
\\
&\le&\log(m(n))-M\frac{2^{-K}}{1-2^{-K}}
\\
&\le&\log(m(n))-M2^{-{(K+1)}}
\\
&\le&\log(m(n))-2\log(m(n))=\log\left(\frac{1}{m(n)}\right)
\end{eqnarray*}
\end{proof}

\subsection{From $\csp_{n^{d}}^{\rand}(T_{K,M},\neg T_{K,M})$ to DNF's}
\begin{lemma}\label{lem:double_T_to_DNF}
Suppose that $P:\{\pm 1\}^K\to \{0,1\}$ can be realized a $\dnf$ formula with $T$ clauses. Then $\ch_P$ can be efficiently realized\footnote{That is, there is an efficiently computable $g:\cx_{n,K}\to \{\pm 1\}^{2Kn}$ for which each $h\in \ch_P$ is of the form $h=h'\circ g$ for some DNF formula $g$ with $T$ clauses and $2Kn$ variables.} by the class of $\dnf$ formulas with $T$ clauses and $2Kn$ variables.
\end{lemma}
\proof
The realization is defined by the function $g:\cx_{n,K}\to \{\pm 1\}^{2Kn}$, defined as follows. We will index the coordinates of vectors in $\{\pm 1\}^{2Kn}$ by $[K]\times \{\pm 1\}\times[n]$ and let
\[
g_{j,b,i}(x)=
\begin{cases}
-1 & x(j)=(-b,i)
\\
1 & \text{otherwise}
\end{cases}~~.
\]
To see that $g$ indeed defines a realization of $\ch_P$ by the class of $\dnf$ formulas with $T$ clauses, we must show that for any assignment $\psi\in \{\pm 1\}^n$, $h_\psi=h\circ g$ for some 
$\dnf$ formula $h$ with $T$ clauses.

Indeed, write $P(z_1,\ldots,z_K)=\vee_{t=1}^T\wedge_{r=1}^{R_{t}}b_{t,r}z_{j_{t,r}}$ for $b_{t,r}\in \{\pm 1\}$ and $i_{t,r}\in [K]$.  Now consider the formula $h:\{\pm 1\}^{2Kn}\to \{0,1\}$ defined by
\[
h(x)=\vee_{t=1}^T\wedge_{r=1}^{R_t}\wedge_{i=1}^{n}x_{j_{t,r},\psi_{i}b_{t,r},i}~.
\]
Now, for $x\in \cx_{n,K}$ we have,
\begin{align*}
h(g(x))=1 ~~&\iff \exists t\in[T]\,\forall r\in[R_t], i\in [n],\;
g_{j_{t,r},\psi_{i}b_{t,r},i}(x)=1
\\
&\iff  \exists t\in[T]\,\forall r\in[R_t], i\in [n],\;x(j_{t,r})\ne(-\psi_{i}b_{t,r},i)
\\
&\iff \exists t\in[T]\,\forall r\in[R_t],\; \;x_1(j_{t,r})\ne -\psi_{x_2(j_{t,r})}b_{t,r}
\\
&\iff \exists t\in[T]\,\forall r\in[R_t],\; x_1(j_{t,r})\psi_{x_2(j_{t,r})}=b_{t,r}
\\
&\iff h_{\psi}(x)=x(\psi)=P(x_1(1)\psi_{x_2(1)},\ldots,x_1(K)\psi_{x_2(K)})=1  ~.
\end{align*}
\proofbox

\subsection{Wrapping up -- concluding theorem \ref{thm:dnf_few_clauses}}
We are now ready to conclude the proof.
Let $q:\mathbb N\to\mathbb N$ be any function such that $q(n)=\omega(\log(n))$. W.l.o.g., we assume that $q(n)=O\left(\log^2(n)\right)$. 
By theorem \ref{thm:basic_realizable} it is enough to show that for every $d$, it is hard to distinguish samples that are realizable by $\dnf^{q(n)}$ and $\left(n^{d},1/4\right)$-scattered samples.

By assumption \ref{hyp:only_sat}, there is $K$ such that $\csp_{n^{d+2}}^{\rand}(\sat_K)$ is hard. 
Denote $q'(n)=q(2Kn)$.
By lemma \ref{lem:sat_to_T}, the problem $\csp_{n^{d+1}}^{\rand}(T_{K,q'(n)})$ is hard. By lemma \ref{lem:T_to_double_T}, the problem $\csp_{n^{d+1}}^{\rand}(T_{K,q'(n)},\neg T_{K,q'(n)})$ is hard. Now, since $\neg T_{K,q'(n)}$ can be realized by a DNF formula with $q'(n)$ clauses, by lemma \ref{lem:double_T_to_DNF}, the problem $\csp_{n^{d+1}}^{\rand}(T_{K,q'(n)},\neg T_{K,q'(n)})$ can be reduced to a problem of distinguishing samples that are realizable by a DNF formula with $2Kn$ variables and $q'(n)$ clauses, from $\left(\frac{1}{8}n^{d+1},1/4\right)$-scattered samples. Changing variables (i.e., replacing $2Kn$ with $n'$), we conclude that it is hard to distinguish samples that are realizable by $\dnf^{q\left(n\right)}$ from $\left(\frac{1}{8(2K)^{d-1}}n^{d+1},1/4\right)$-scattered samples, which are in particular $\left(n^{d},1/4\right)$-scattered. The theorem follows.

\section{Proof theorem \ref{thm:auto}}\label{sec:auto}
For a function $q:\mathbb N\to\mathbb N$, we let $\auto_{q(n)}$ be the
class of functions $h:\{\pm 1\}^n\to \{0,1\}$ that can be realized by
a finite automaton with $\le q(n)$ states. We will show that $\dnf^{n}$ can be efficiently realized to $\auto_{2n+1}$. By theorem \ref{thm:dnf_few_clauses} it follows that learning $\auto_{2n+1}$ is RSAT-hard. By a simple scaling argument (see \cite{daniely2013average}), the same conclusion holds for $\auto_{n^{\epsilon}}$.

The realization is very simple, and is given by the mapping $g:\{\pm \}^n\to \{\pm 1\}^{n^2}$, where $g(x)$ is simply $n$ consecutive copies of $x$. Let $T:\{\pm 1\}^n\to\{0,1\}$ be a function that can be realized by a $\dnf$ formula consisting of $n$ clauses. We must show that the is an automaton $A:\{\pm 1\}^{n^2}\to\{0,1\}$, with $\le 2n^2+1$ states, such that $T=A\circ g$.

The automaton $A$ will have an accepting sink, and two states to each input variable. Given $z=g(x)$, the automaton will first go over the first $n$ inputs, and will check whether there is violation of the first clause in the $\dnf$ representation of $T$. This can be easily done with two states for each variable -- the first state indicate that there was no violation till this state, while the second will indicate the opposite. After going over the first $n$ inputs, if there was no violation, the automaton will jump to the accepting sink. Otherwise, it will continue the next $n$ variables, and will similarly check whether there is a violation of the second clause in the $\dnf$ representation of $T$. Again, if there was no violation, the automaton will jump to the accepting sink, and otherwise, will continue the the next $n$ inputs, checking the third clause. If, after checking all the $n$ clauses, it turns out that all clauses are violated, the automaton will reject.

\section{Proof of theorem \ref{thm:basic_realizable} \cite{daniely2013average}}\label{sec:proof_thm_basic}
Let $\ch$ be the hypothesis class in question and suppose toward a contradiction that algorithm $\cl$ learns $\ch$ efficiently. Let $M\left(n,1/\epsilon,1/\delta\right)$ be the maximal number of random bits used by $\cl$ when it run on the input $n,\epsilon,\delta$. This includes both the bits describing the examples produced by the oracle and ``standard" random bits. Since $\cl$ is efficient, $M\left(n,1/\epsilon,1/\delta\right)< \poly(n,1/\epsilon, 1/\delta)$. Define
\[
q(n)=M\left(n,1/\beta,4\right)+n~.
\]
By assumption, there is a $(q(n),\beta)$-scattered ensemble $\cd$
for which it is hard to distinguish a $\cd$-random sample from a
realizable sample. Consider the algorithm $\ca$ defined below. On input $S\in\cz_n^{m(n)}$,
\begin{enumerate}
\item Run $\cl$ with parameters $n,\beta$ and $\frac{1}{4}$, such that the examples' oracle generates examples by choosing a random example from $S$.
\item Let $h$ be the hypothesis that $\cl$ returns. If $\Err_S(h)\le \beta$, output $\text{``realizable"}$. Otherwise, output $\text{``unrealizable"}$.
\end{enumerate}
Next, we derive a contradiction by showing that $\ca$ distinguishes a realizable sample from a $\cd$-random sample. Indeed, if the input $S$ is realizable, then $\cl$ is guaranteed to return, with probability $\ge 1-\frac{1}{4}$, a hypothesis $h:\cx_n\to\{0,1\}$ with $\Err_{S}(h)\le \beta$. Therefore, w.p. $\ge \frac{3}{4}$ $\ca$ will output ``realizable".

What if the input sample $S$ is drawn from $\cd^{m(n)}_n$? Let
$\cg\subset \{0,1\}^{\cx_n}$ be the collection of functions that $\cl$
might return when run with parameters $n,\epsilon(n)$ and
$\frac{1}{4}$. We note that $|\cg |\le 2^{q(n)-n}$, since  each
hypothesis in $\cg$ can be described by $q(n)-n$ bits. Namely, the
random bits that $\cl$ uses and the description of the examples
sampled by the oracle. Now, since $\cd$ is
$(q(n),\beta)$-scattered, the probability that $\Err_{S}(h)\le
\beta$ for some $h\in \cg$ is at most $|\cg|2^{-q(n)}\le
2^{-n}$. It follows that the probability that $\ca$ responds
``realizable" is $\le 2^{-n}$. This leads to the desired contradiction
and concludes our proof.

\section{Open questions}\label{sec:future}
An obvious direction for future work is to establish more lower bounds. We list below some basic learning problems that we are unable to resolve even 
under the random $K$-SAT assumption.
\begin{enumerate}
\item Learning decision trees.
\item Learning intersections of a constantly many halfspaces. It is worth noting that no known algorithm can learn even intersections of $2$ halfspaces.
\item Agnostically Learning halfspaces with a constant approximation ratio. 
We note that the last problem was shown hard under the much stronger assumption of \cite{daniely2013average}.
\end{enumerate}
In addition, as discussed in \cite{daniely2013average}, our work and \cite{daniely2013average} have connections to several TCS areas, including hardness of approximation, cryptography, refutation algorithms and average case complexity.

\paragraph{Acknowledgements:}
Amit Daniely is a recipient of the Google Europe Fellowship in Learning Theory, and this research is supported in part by this Google Fellowship. Shai Shalev-Shwartz is supported by the Israeli Science Foundation grant number 590-10. 
We thank Uri Feige, Guy Kindler and Nati Linial for valuable discussions.

\bibliography{bib}

\end{document}